\documentclass[10pt,letterpaper]{article}
\usepackage[letterpaper,margin=1in]{geometry}
\usepackage[T1]{fontenc}
\usepackage[utf8]{inputenc}
\usepackage{amsmath,amsthm,mathtools}
\usepackage{amssymb}
\usepackage{newtxtext,newtxmath}
\usepackage{graphicx,booktabs,array,tabularx}
\usepackage[table,dvipsnames]{xcolor}
\usepackage{enumitem,microtype}
\usepackage[font=small,labelfont=bf]{caption}
\usepackage{url}
\usepackage[colorlinks=true,citecolor=blue,linkcolor=blue,urlcolor=blue]{hyperref}
\usepackage{bookmark}
\newtheorem{theorem}{Theorem}
\newtheorem{proposition}{Proposition}

\theoremstyle{remark}
\newtheorem{remark}{Remark}

\theoremstyle{definition}
\newtheorem{definition}{Definition}

\newcommand{\Softplus}{\operatorname{Softplus}}
\newcommand{\arccosh}{\operatorname{arccosh}}

\newcommand{\SO}{\mathrm{SO}}
\newcommand{\so}{\mathfrak{so}}
\DeclareMathOperator{\diag}{diag}

\setlist{nosep,topsep=4pt}
\graphicspath{{figs/}}

\hypersetup{
    colorlinks=true,
    linkcolor=blue,
    citecolor=blue,
    urlcolor=blue
}

\title{\textbf{Minkowski Attractor Networks: \\Closed-Form Hyperbolic Flows for Visual Representations}}

\author{
    \texttt{\;}\\
    \textbf{Zhongping Ji}
}

\date{}

\begin{document}

\maketitle


\begin{abstract}
Geometric representation learning predominantly scaffolds representations onto flat Euclidean subspaces or compact product tori ($\mathbb{T}^K$). However, flat manifolds possess vanishing curvature and polynomial volume growth, inherently suffering from metric distortion when embedding multi-scale, tree-like visual hierarchies. While hyperbolic spaces ($\mathbb{H}^m$) circumvent this via constant negative curvature ($K<0$) and exponential volume expansion, prior hyperbolic deep architectures are hindered by computationally cumbersome Riemannian optimization, non-linear gyrovector calculus, and floating-point instabilities.

In this work, we introduce \textbf{Minkowski Attractor Networks (MAN)}, an operator-splitting-inspired framework that embeds representations within pseudo-Riemannian Minkowski spacetime ($\mathbb{R}^{1,m}$). By framing hyperbolic manifolds as quadric level sets, MAN resolves hyperbolic geometry by combining linear Lorentz group transport with non-linear cone lifting and closed-form radial rescaling, evaluating in a single forward pass without numerical ODE solvers or iterative retractions. 
We establish \textbf{MAN-2D} ($\mathbb{R}^{1,1} \to \mathbb{H}^1$) as our primary, high-throughput visual backbone, which maximizes channel factorization granularity into $D/2$ independent two-dimensional Minkowski blocks. We further formulate \textbf{MAN-4D} ($\mathbb{R}^{1,3} \to \mathbb{H}^3$) as a spacetime extension, leveraging a commuting Cartan-subalgebra parameterization of $\mathrm{SO}^+(1,3)$ to evaluate 4D Lorentz isometries via two commuting 2D planar maps without matrix-exponential overhead.
Across both compact ($\sim 1\text{M}$) and scaled ($\sim 2.1\text{M}$) regimes on CIFAR-100 without external pretraining, MAN models establish new Pareto frontiers: MAN-2D-1 achieves 81.03\% top-1 accuracy (1.02M parameters), scaling to 81.82\% in MAN-2D-2 (2.13M parameters). All variants comprehensively outperform flat torus baselines and heavyweight 23.7M ResNet-50, demonstrating that closed-form Minkowski spacetime dynamics provide a superior foundation for visual representations. Code will be made publicly available upon acceptance at: https://github.com/ParaMind2025/Ananke-CV.
\end{abstract}


\section{Introduction}
Modern deep visual architectures predominantly cascade discrete linear and non-linear operations within flat Euclidean vector spaces \cite{he2016deep,dosovitskiy2020image,liu2022convnet}. Continuous-depth models \cite{chen2018neural} interpret neural cascades as dynamical flow systems, while structured geometric backbones, such as Contractive Torus Attractor Networks (CTAN) \cite{ananke2026}, demonstrate that high-dimensional latent spaces can be structured into direct sums of 2D phase planes ($\bigoplus \mathbb{R}^2$), where transverse perturbations contract onto compact invariant tori $\mathbb{T}^K$.

Despite their parameter efficiency, product-torus models face an intrinsic geometric limitation: \textbf{flat tori have vanishing sectional curvature}. In flat spaces, the volume of a geodesic ball scales polynomially ($\operatorname{Vol}(B(r)) \propto r^m$), saturating at large scales on compact manifolds. In contrast, natural visual concepts exhibit multi-scale, tree-like taxonomies (e.g., compositional parts, fine-grained species, and hierarchical scene semantics). Embedding discrete branching trees into fixed low-dimensional flat spaces inevitably incurs metric distortion unless dimension scales with tree depth \cite{nickel2017poincare,ganea2018hyperbolic}.

Hyperbolic geometry ($\mathbb{H}^m$, for $m \ge 2$) offers an escape from this dimensional bottleneck. Endowed with constant negative curvature ($K < 0$), hyperbolic space exhibits exponential metric volume expansion ($\operatorname{Vol}(B(r)) \sim e^{(m-1)r/R}$), providing natural geometric capacity for continuous tree-like data. Nonetheless, mainstream computer vision has remained resistant to hyperbolic models due to three fundamental obstacles:
\begin{enumerate}
    \item \textbf{Non-Linear Gyrocalculus:} Standard operations (vector additions, residual shortcuts) are undefined in hyperbolic geometry, forcing prior works into computationally expensive gyrovector operations (M\"obius addition) or iterative Fr\'echet means \cite{ganea2018hyperbolic}.
    \item \textbf{Optimization and Numerical Challenges:} Directly optimizing manifold-constrained hyperbolic embeddings typically requires Riemannian gradient updates, whereas unconstrained network weights can be trained with standard Euclidean optimizers. Separately, finite-precision arithmetic limits the numerical reliability of hyperbolic representations in both the Poincar\'e and Lorentz models, potentially causing loss of accuracy or NaN failures \cite{mishne2023numerical}. These challenges motivate formulations with explicit geometric constraints and tractable computations.
    \item \textbf{GPU Hardware Incompatibility:} Modern Tensor Cores are specialized for Dense General Matrix Multiply (GEMM). Hyperbolic projections fragment computation into memory-bound elementwise operations.
\end{enumerate}

\paragraph{Why Minkowski Spacetime?} 
We overcome these obstacles by demonstrating that \textbf{embedding representations into pseudo-Riemannian Minkowski spacetime ($\mathbb{R}^{1,m}$) linearizes hyperbolic dynamics}. 
In Minkowski spacetime endowed with indefinite metric $\eta = \diag(1, -1, \dots, -1)$, the hyperbolic manifold is realized simply as a quadric level set:
\begin{equation}
\mathbb{H}_R^m = \left\{ X \in \mathbb{R}^{1,m} \ : \ X_0^2 - \sum_{i=1}^m X_i^2 = R^2, \ X_0 > 0 \right\}.
\end{equation}
Because the ambient spacetime is flat (pseudo-Euclidean), orientation-preserving hyperbolic isometries (Lorentz boosts and rotations) are represented as \textbf{standard linear matrix multiplications} belonging to the Lie group $\mathrm{SO}^+(1,m)$. 
Coupled with Lie--Trotter operator splitting, spatial feature mixing is handled by standard depthwise convolutions, while local reaction dynamics evolve via an exact, single-pass closed-form mapping with no time-discretization truncation error in exact arithmetic for the local reaction subproblem.

From a representation-learning perspective, the \textbf{manifold hypothesis} motivates representations organized near lower-dimensional geometric structures, while the risk of \textbf{representation collapse} highlights the need to retain informative variation within those structures. These competing imperatives motivate a principled separation between transverse regularization and tangential transport. Ananke \cite{ananke2026} implements this separation through structured latent manifold priors, combining logarithmic radial relaxation with conservative transport along their level sets. For the exact MAN reaction flow with shared, frozen parameters, the log-radius error decays exponentially, while Lorentz transport preserves pairwise hyperbolic distances between normalized states. This provides a conditional mechanism for regulating radial deviations without contracting existing intrinsic differences. CTAN and MAN instantiate this design using circular and hyperbolic geometries, respectively. These guarantees establish a rigorous local certificate for the conditional flows, providing the broader learned architecture with a structured geometric prior to attenuate off-manifold noise while safeguarding semantic expressivity.

\paragraph{Our Contributions.}
\begin{itemize}
    \item \textbf{The Minkowski Attractor Prior:} We establish an operator-splitting-inspired continuous-dynamical framework grounded in Minkowski spacetime. Transverse signed logarithmic dissipation contracts off-manifold perturbations, while tangential Lorentz Lie group flows preserve metric invariants.
    \item \textbf{MAN-2D as the Workhorse Visual Backbone:} We show that the 2D Minkowski formulation ($\mathbb{R}^{1,1} \to \mathbb{H}^1$) is an exceptional visual operator. Although $\mathbb{H}^1$ is intrinsically flat, its non-compact coordinates prevent periodic phase-locking, while partitioning channels into $D/2$ independent 2D blocks maximizes filter factorization granularity.
    \item \textbf{High-Dimensional Spacetime Extension (MAN-4D):} We generalize the framework to four-dimensional spacetime ($\mathbb{R}^{1,3} \to \mathbb{H}^3$). By exploiting the commuting Cartan-subalgebra parameterization $[\mathbf{K}_3, \mathbf{J}_{12}] = 0$, we decouple 4D Lorentz transport into parallel spatial rotations and longitudinal boosts, preserving the Lorentz quadratic form through a two-parameter commuting subgroup at reduced parameter overhead ($1.75D$ vs. $2.0D$).
    \item \textbf{Systematic Dimensional and Scaling Benchmark:} We evaluate 2D, 3D, and 4D variants across two pyramidal tiers on CIFAR-100. MAN models establish new performance frontiers: MAN-2D attains \textbf{81.03\%} (1.02M) and \textbf{81.82\%} (2.13M); MAN-4D attains \textbf{80.80\%} (0.97M) and \textbf{81.75\%} (2.01M), universally outperforming 2D flat tori (78.41\% / 80.54\%) and heavyweight 23.7M ResNet-50 (79.14\%).
\end{itemize}


\section{Related Work}
\paragraph{Continuous-Depth Representations.} Neural ODEs \cite{chen2018neural} formalized residual networks as continuous dynamical systems. Hamiltonian Neural Networks \cite{greydanus2019hamiltonian} and Symplectic ODE-Nets \cite{zhong2020symplectic} integrated conservative physical invariants. Operator splitting techniques, such as Lie--Trotter and Strang splitting \cite{strang1968construction}, separate spatial diffusion from local non-linear kinetics. CTAN \cite{ananke2026} operationalized this via 2D flat tori. MAN advances this continuous paradigm into non-Euclidean pseudo-Riemannian spacetimes.

\paragraph{Hyperbolic Deep Learning.} Hyperbolic geometry was introduced to deep learning via Poincar\'e embeddings for hierarchical structures \cite{nickel2017poincare} and generalized through gyrovector calculus \cite{ganea2018hyperbolic}. Hyperbolic graph networks \cite{chami2019hyperbolic} and continuous Riemannian normalizing flows \cite{mathieu2020riemannian} further demonstrated representation gains. Prior works have developed fully hyperbolic neural networks directly in the Lorentz model \cite{chen2021fully,bdeir2023fully}.
MAN instead introduces continuous reaction dynamics with closed-form logarithmic radial relaxation and low-dimensional structured Lorentz transport.


\section{Foundations: Minkowski Spacetime and the Hyperboloid Model}

\subsection{Pseudo-Riemannian Metric and Causal Cones}
Let $\mathbb{R}^{1,m}$ denote $(m+1)$-dimensional Minkowski spacetime. A state vector is partitioned into a timelike scalar $X_0 \in \mathbb{R}$ (hierarchy depth) and spacelike coordinates $\mathbf{X} = (X_1, \dots, X_m)^\top \in \mathbb{R}^m$ (branching directions):
\begin{equation}
X = \begin{pmatrix} X_0 \\ \mathbf{X} \end{pmatrix} \in \mathbb{R}^{1,m}.
\end{equation}
Spacetime is endowed with the diagonal Minkowski metric tensor $\eta = \diag(1, -1, \dots, -1)$. The pseudo-inner product is defined by:
\begin{equation}
\langle X, Y \rangle_\eta = X^\top \eta Y = X_0 Y_0 - \mathbf{X} \cdot \mathbf{Y}.
\end{equation}
The Minkowski squared pseudo-norm is $\|X\|_\eta^2 = \langle X, X \rangle_\eta = X_0^2 - \|\mathbf{X}\|^2$. A state vector is timelike if $\|X\|_\eta^2 > 0$, lightlike if $\|X\|_\eta^2 = 0$, and spacelike if $\|X\|_\eta^2 < 0$. The future timelike cone is defined as $\mathcal{C}^+ \coloneqq \{X \in \mathbb{R}^{1,m} : X_0 > 0, \langle X, X \rangle_\eta > 0\}$.

\subsection{The Lorentz Hyperboloid $\mathbb{H}_R^m$}
\begin{definition}[Hyperboloid Model]
The $m$-dimensional hyperbolic space of radius $R > 0$ is realized as the forward sheet of the two-sheeted hyperboloid in $\mathbb{R}^{1,m}$:
\begin{equation}
\mathbb{H}_R^m \coloneqq \left\{ X \in \mathcal{C}^+ \ : \ \langle X, X \rangle_\eta = R^2 \right\}.
\end{equation}
\end{definition}

\begin{proposition}[Induced Positive-Definite Metric]
Although $\eta$ is indefinite on $\mathbb{R}^{1,m}$, the negative restriction of $\eta$ to the tangent bundle of $\mathbb{H}_R^m$ induces a strictly positive-definite Riemannian metric $g = -\iota^* \eta$.
\end{proposition}
\begin{proof}
For any $X \in \mathbb{H}_R^m$, the tangent space is $T_X \mathbb{H}_R^m = \{V \in \mathbb{R}^{1,m} : \langle X, V \rangle_\eta = 0\}$. Because $X$ is strictly timelike ($X_0 > \|\mathbf{X}\|$), any non-zero vector $V$ orthogonal to $X$ under $\eta$ must be strictly spacelike ($\langle V, V \rangle_\eta < 0$). Thus, $g(V, V) \coloneqq -\langle V, V \rangle_\eta > 0$ for all non-zero $V \in T_X \mathbb{H}_R^m$.
\end{proof}

\begin{theorem}[Sectional Curvature and Volume Explosion]
For $m \ge 2$, the manifold $\mathbb{H}_R^m$ possesses constant negative sectional curvature $K = -1/R^2 < 0$. The volume of a geodesic ball of radius $r$ scales exponentially:
\begin{equation}
\operatorname{Vol}_m(B(r)) \sim \frac{S_{m-1} R^m}{(m-1)2^{m-1}} e^{\frac{(m-1)r}{R}} \quad (r \to \infty).
\end{equation}
For $m=1$ ($\mathbb{H}^1$), the intrinsic curvature tensor vanishes identically, and $\operatorname{Vol}_1(B(r)) = 2r$.
\end{theorem}

\begin{definition}[Intrinsic Geodesic Distance]
The geodesic distance between $X, Y \in \mathbb{H}_R^m$ is:
\begin{equation}
d_{\mathbb{H}}(X, Y) = R \arccosh\left( \frac{\langle X, Y \rangle_\eta}{R^2} \right).
\end{equation}
By the reversed Cauchy--Schwarz inequality for future-directed timelike vectors, $\langle X, Y \rangle_\eta \ge R^2$, guaranteeing that the argument to $\arccosh$ is strictly $\ge 1$.
\end{definition}


\section{Methodology: Minkowski Attractor Networks}

\subsection{Recap: The Product-Manifold Scaffolding Prior}
\label{subsec:scaffolding_prior}

Rather than presuming that high-dimensional sensory representations are globally homeomorphic to a single monolithic manifold, the foundational Ananke framework \cite{ananke2026} introduces the \textbf{Product-Manifold Scaffolding Prior}. In analogy to harmonic analysis where orthogonal circular functions serve as a universal basis to decompose complex signals, the latent representation space $\mathbb{R}^D$ is factorized into an orthogonal direct sum of low-dimensional elementary phase subspaces:
\begin{equation}
X(\mathbf{p}) = \bigoplus_{k=1}^M X_k(\mathbf{p}), \qquad X_k(\mathbf{p}) \in \mathbb{R}^{d_{\text{sub}}},
\label{eq:direct_sum_scaffold}
\end{equation}
where $X \in \mathbb{R}^{B \times D \times H_s \times W_s}$ denotes the visual feature tensor, $\mathbf{p} = (h, w)$ indexes the spatial lattice, and $D = M \cdot d_{\text{sub}}$. Here, the orthogonal direct sum serves as an adaptive geometric scaffold: each constituent subspace provides an elementary coordinate frame where transverse amplitude deviations and tangential semantic transport can be regulated independently. Across the spatial grid, these local fibers are conditioned on the input feature context, dynamically adapting the scaffold to the underlying data distribution.

Within each constituent subspace, representations evolve under a conditional reaction flow toward a prescribed target level-set manifold:
\begin{equation}
\dot{X}_k = F_k(X_k; \Theta_k), \qquad \mathcal{M}_\Theta = \prod_{k=1}^M \left\{ X_k \ : \ H_k(X_k; \Theta_k) = 0 \right\},
\label{eq:product_manifold_def}
\end{equation}
where $\Theta_k$ denotes the parameter collection computed once at the block input and held frozen during the local reaction step to ensure decoupled parallel planar flows, and $H_k$ specifies the energy function defining the invariant submanifold.

\paragraph{From Flat Toroids to Minkowski Spacetimes.}
In CTAN \cite{ananke2026}, this scaffolding was instantiated via canonical 2D Euclidean phase planes ($d_{\text{sub}} = 2$) with circular potential level sets $H_k(X_k) = \frac{1}{2}(\|X_k - c_k\|_2^2 - R_k^2) = 0$, whose Cartesian product forms a compact invariant torus $\mathcal{M}_\Theta \cong \mathbb{T}^K$. In this work, MAN advances this scaffolding prior by lifting the constituent fibers into pseudo-Riemannian Minkowski spacetimes ($d_{\text{sub}} \in \{2, 4\}$). Instead of compact circular orbits, representations evolve toward invariant hyperbolic quadrics $H_k(X_k) = \frac{1}{2}(\|X_k - c_k\|_\eta^2 - R_k^2) = 0$, preserving the decoupled solvability of the original scaffolding prior while unlocking non-compact, hierarchical geometric capacity.

\subsection{Operator-Split Reaction-Diffusion Flow}
Let $X \in \mathbb{R}^{B \times D \times H_s \times W_s}$ denote a feature tensor. We factorize the channels into an orthogonal direct sum of independent Minkowski blocks $\bigoplus_{k=1}^M \mathbb{R}^{1,m}$. Feature evolution follows an operator-splitting-inspired pipeline:
\begin{enumerate}
    \item \textbf{Diffusion Step:} Multi-scale spatial depthwise convolution $S$ captures neighbor context: $\tilde{X} = S(X_l)$.
    \item \textbf{Reaction Step:} The closed-form Minkowski flow operator $\Psi_{\Delta t}$ acts locally within each Minkowski subspace: $X_{l+1} = \Psi_{\Delta t}(\tilde{X})$.
\end{enumerate}

\subsection{Continuous Dynamics and Orthogonal Decomposition}
\label{subsec:continuous_lorentzian_dynamics}

Let $u(t) = X(t) - c \in \mathcal{C}^+$ denote the relative state centered at a learned origin $c$. Within each reaction step, parameters $\Theta = (c, R, \beta, \Omega)$ are held frozen across the local integration interval. For any timelike state with Minkowski radius $r_L(u) = \sqrt{\langle u, u \rangle_\eta} > 0$, the ambient space factorizes into a direct sum of the normal ray and the tangent bundle to the level set $\mathcal{M}_{r_L} = \{v : \langle v, v \rangle_\eta = r_L(u)^2\}$:
\begin{equation}
T_u \mathbb{R}^{1,m} = \mathcal{N}_u \oplus \mathcal{T}_u,
\end{equation}
where $\mathcal{N}_u \coloneqq \operatorname{span}\{u\}$ and $\mathcal{T}_u \coloneqq \{ v : \langle u, v \rangle_\eta = 0 \} = T_u \mathbb{H}_{r_L}^m$.

We formulate the continuous reaction vector field as an orthogonal decomposition:
\begin{equation}
\frac{du}{dt} = F_{\text{local}}(u; \Theta) = \underbrace{F_{\perp}(u)}_{\text{Signed Normal Dissipation}} + \underbrace{F_{\parallel}(u)}_{\text{Tangential Isometry Flow}},
\label{eq:orthogonal_decomposition}
\end{equation}
where:
\begin{equation}
\left\{
\begin{aligned}
F_{\perp}(u) &= -\beta \ln\left(r_L(u)/R\right) u, \quad \beta > 0, \\[0.6180em]
F_{\parallel}(u) &= \Omega u, \quad \Omega \in \mathfrak{so}(1,m), \quad \Omega^\top \eta + \eta \Omega = 0.
\end{aligned}
\right.
\end{equation}

Specifically, the normal dissipative component $F_{\perp}(u) \in \mathcal{N}_u$ acts parallel to the radial position vector $u$, exerting a scale-invariant logarithmic restoring force that contracts transverse off-manifold perturbations normally toward the target hyperboloid $\mathbb{H}_R^m$ at an exponential rate (with noise robustness certified by the radial perturbation bound in Proposition~\ref{prop:radial_error_bound}). Concurrently, the tangential isometric component $F_{\parallel}(u) \in \mathcal{T}_u$ is governed by a velocity generator $\Omega \in \mathfrak{so}(1,m)$, acting strictly within the tangent bundle $T_u \mathbb{H}_{r_L}^m$ to satisfy $\langle u, F_{\parallel}(u) \rangle_\eta \equiv 0$, which executes coordinate transport inside the manifold (hierarchical depth traversal and semantic branch switching) while preserving the Minkowski pseudo-norm identically.

Along a continuous trajectory $u(t)$ with initial state $u_0 = u(0)$, we abbreviate $r_L(t) \coloneqq r_L(u(t))$ and $r_L(0) \coloneqq r_L(u_0)$.
Defining the logarithmic transverse error coordinate $z(t) \coloneqq \ln(r_L(u(t))/R)$, we obtain the following exact properties:

\begin{proposition}[Orthogonal Decoupling and Exact Solution]
Along trajectories of \eqref{eq:orthogonal_decomposition} with frozen parameters $\Theta$, tangential Lorentz flows make identically zero contribution to the Minkowski norm, leaving the radial dynamics governed purely by normal logarithmic dissipation:
\begin{equation}
\frac{d}{dt} \|u\|_\eta^2 = -2\beta \ln\left(\frac{r_L}{R}\right) \|u\|_\eta^2 \iff \dot{r}_L = -\beta r_L \ln\left(\frac{r_L}{R}\right).
\end{equation}
For any initial state $u_* \in \mathcal{C}^+$, the exact unique solution across horizon $t \ge 0$ is:
\begin{equation}
\Phi_t(u_*) = \left( \frac{R}{r_L(u_*)} \right)^{1 - e^{-\beta t}} e^{t\Omega} u_*,
\label{eq:exact_solution}
\end{equation}
which preserves $\mathcal{C}^+$ and contracts the logarithmic radial error identically: $z(t) = e^{-\beta t} z(0)$.
\end{proposition}
\begin{proof}
Differentiating: $\frac{d}{dt} \langle u, u \rangle_\eta = 2 u^\top \eta (\Omega u - \beta \ln(r_L/R) u)$. By skew-adjointness of the Lorentz Lie algebra, $\Omega^\top \eta + \eta \Omega = 0$, implying $u^\top \eta \Omega u \equiv 0$. Letting $y = u/r_L \in \mathbb{H}_1^m$, we have $\dot{y} = \Omega y$, hence $y(t) = e^{t\Omega} y(0)$. Multiplying by $r_L(t) = R \cdot \left(r_L(u_0)/R\right)^{e^{-\beta t}}$ yields \eqref{eq:exact_solution}.
\end{proof}

\begin{proposition}
[Semicontraction under Natural Product Metric]
\label{prop:radial_error_bound}
On the forward timelike cone $\mathcal{C}^+$, define the product metric for states $u, v \in \mathcal{C}^+$:
\begin{equation}
d_*^2(u, v) \coloneqq \left( \ln\frac{r_L(u)}{r_L(v)} \right)^2 + d_{\mathbb{H}_1^m}^2\left( \frac{u}{r_L(u)}, \frac{v}{r_L(v)} \right).
\end{equation}
For two trajectories of the same conditional flow with shared, fixed parameters $R>0, \beta>0, \Omega$, the flow is non-expansive under $d_*$, with exponential contraction of the logarithmic radial component:
\begin{equation}
d_*^2(\Phi_t(u), \Phi_t(v)) = e^{-2\beta t} \left( \ln\frac{r_L(u)}{r_L(v)} \right)^2 + d_{\mathbb{H}_1^m}^2\left( \frac{u}{r_L(u)}, \frac{v}{r_L(v)} \right).
\end{equation}
This property formalizes the stability of the conditional dynamics under fixed parameters, while data-driven contextual modulation across distinct inputs accommodates global semantic expressivity.
\end{proposition}

\subsection{Cone Lifting and End-to-End Layer Definition}
Arbitrary neural activations can land outside the forward cone $\mathcal{C}^+$. To ensure numerical robustness, we define the non-linear \textbf{cone lifting map} $P_\varepsilon: \mathbb{R}^{1,m} \to \mathcal{C}^+$:
\begin{equation}
P_\varepsilon(v) \coloneqq \begin{pmatrix} \sqrt{\|\mathbf{v}\|_2^2 + \Softplus(v_0)^2 + \varepsilon} \\ \mathbf{v} \end{pmatrix}, \quad \varepsilon > 0.
\end{equation}
In exact arithmetic, $\|P_\varepsilon(v)\|_\eta^2 \equiv \Softplus(v_0)^2 + \varepsilon > 0$, guaranteeing a timelike state. The complete network layer evaluates as cone lifting followed by the exact conditional flow:
\begin{equation}
\Psi_{\Delta t}(X) = c + \Phi_{\Delta t}\big(P_\varepsilon(X - c)\big).
\end{equation}

\subsection{Primary Workhorse: MAN-2D Architecture}
\label{subsec:man_2d}
In \textbf{MAN-2D}, the latent space is partitioned into $K = D/2$ independent two-dimensional Minkowski phase planes $\bigoplus_{k=1}^K \mathbb{R}^{1,1}$, with state $X = (X_0, X_1)^\top \in \mathbb{R}^{1,1}$ and $\eta = \diag(1, -1)$.

In 2D spacetime, spatial rotations do not exist ($\SO(1) \cong \{1\}$). The restricted Lorentz group $\SO^+(1,1)$ consists purely of 1D hyperbolic boosts parameterized by rapidity $\varphi \in \mathbb{R}$. With generator $\mathbf{K}_1 = \begin{pmatrix} 0 & 1 \\ 1 & 0 \end{pmatrix}$, setting $\Omega = \frac{\varphi}{\Delta t}\mathbf{K}_1$ yields \footnote{Note on convention: In contrast to the passive coordinate-transformation convention commonly seen in relativity textbooks (which carries a negative sign on $\sinh\varphi$), we adopt the active transformation convention standard in Lie group dynamical flows, where $\exp(\varphi \mathbf{K})$ actively advances the latent state forward along the positive spatial direction.}:
\begin{equation}
\begin{pmatrix} u_0' \\ u_1' \end{pmatrix} = \begin{pmatrix} \cosh\varphi & \sinh\varphi \\ \sinh\varphi & \cosh\varphi \end{pmatrix} \begin{pmatrix} u_0 \\ u_1 \end{pmatrix}.
\label{eq:2d_boost}
\end{equation}

The complete MAN-2D mapping evaluates as:
\begin{equation}
\Psi_{\Delta t}^{\text{2D}}(X) = c + \rho(\Delta t) \cdot \begin{pmatrix} u_0 \cosh\varphi + u_1 \sinh\varphi \\ u_0 \sinh\varphi + u_1 \cosh\varphi \end{pmatrix},
\end{equation}
where $u = P_\varepsilon(X-c)$, and $\rho(\Delta t) = (R / r_L(u))^{1 - e^{-\beta \Delta t}}$.

\paragraph{Geometric Nature and Granularity.}
MAN-2D does not obtain an intrinsic negative-curvature advantage: each individual factor $\mathbb{H}_R^1$ is isometric to the real line $(\mathbb{R}, R^2 d\phi^2)$, and their product $(\mathbb{H}^1)^K$ is intrinsically flat Euclidean space $\mathbb{R}^K$. However, extrinsically, $\SO^+(1,1)$ provides non-compact coordinates that avoid the periodic phase-locking of compact tori ($\mathbb{S}^1$). Furthermore, because channels factor into $K = D/2$ independent blocks, MAN-2D deploys \textbf{twice as many independent adaptive filters} as 4D ($D/4$ blocks), providing maximal flexibility for dense, low-resolution visual tokens.

\subsection{Spacetime Generalization: MAN-4D Architecture}
\label{subsec:man_4d}
In \textbf{MAN-4D}, channels are partitioned into $M = D/4$ four-dimensional Minkowski blocks $\bigoplus_{m=1}^M \mathbb{R}^{1,3}$ targeting $\mathbb{H}^3$, with state $X = (X_0, X_1, X_2, X_3)^\top$ and $\eta = \diag(1, -1, -1, -1)$.

\paragraph{Full 6-DOF Lorentz Dynamics vs. Computational Trade-Offs.}
The Lie algebra of the restricted Lorentz group $\so(1,3)$ is 6-dimensional, spanned by three spatial rotation generators $\mathbf{J} = (J_{23}, J_{31}, J_{12})$ and three boost generators $\mathbf{K} = (K_1, K_2, K_3)$. 
In principle, a full 6-degree-of-freedom (6-DOF) analytical flow can be evaluated in closed form via the polar decomposition of the Lorentz group:
\begin{equation}
\Lambda(\boldsymbol{\varphi}, \boldsymbol{\vartheta}) = B(\boldsymbol{\varphi}) \begin{pmatrix} 1 & \mathbf{0}^\top \\ \mathbf{0} & R(\boldsymbol{\vartheta}) \end{pmatrix} \in \mathrm{SO}^+(1,3),
\label{eq:full_polar_decomp}
\end{equation}
where $R(\boldsymbol{\vartheta}) \in \mathrm{SO}(3)$ evaluates 3D spatial rotations via Rodrigues' axis-angle formula, and $B(\boldsymbol{\varphi})$ evaluates an unconstrained 3D hyperbolic boost along direction $\mathbf{n} = \boldsymbol{\varphi}/\|\boldsymbol{\varphi}\|$ (detailed in Appendix~\ref{subsec:full_6d_lorentz}). 

However, deploying the unconstrained 6-DOF formulation across high-dimensional feature backbones introduces significant practical trade-offs:
\begin{enumerate}
    \item \textbf{Context Parameter Overhead:} The dynamic context generator must predict 6 independent velocity channels per block instead of 2, widening the parameter projection layer.
    \item \textbf{GPU Memory-Bound Latency:} Evaluating 3D vector cross-products and coordinate projections per block increases on-chip register pressure and intermediate memory traffic.
\end{enumerate}

\paragraph{Commuting Cartan-Subalgebra Parameterization.}
To establish a high-throughput, parameter-compact architecture that balances longitudinal hierarchy with transverse branching, we restrict the velocity generator $\Omega$ to the \textbf{Cartan maximal abelian subalgebra} of $\so(1,3)$. This subalgebra is spanned by the longitudinal boost generator $\mathbf{K}_3$ (rapidity $\varphi$) and the transverse spatial rotation generator $\mathbf{J}_{12}$ (angle $\vartheta$):
\begin{equation}
\Omega = \frac{\varphi}{\Delta t} \mathbf{K}_3 + \frac{\vartheta}{\Delta t} \mathbf{J}_{12} = \frac{1}{\Delta t} \begin{pmatrix}
0 & 0 & 0 & \varphi \\
0 & 0 & -\vartheta & 0 \\
0 & \vartheta & 0 & 0 \\
\varphi & 0 & 0 & 0
\end{pmatrix}, \quad \Delta t > 0.
\end{equation}

Over the finite local reaction horizon $\Delta t$, integrating the continuous generator $\Omega$ evaluates via the matrix exponential, where the time step $\Delta t$ cancels identically:
\begin{equation}
\exp(\Delta t \cdot \Omega) = \exp\left( \Delta t \left[ \frac{\varphi}{\Delta t}\mathbf{K}_3 + \frac{\vartheta}{\Delta t}\mathbf{J}_{12} \right] \right) = \exp(\varphi \mathbf{K}_3 + \vartheta \mathbf{J}_{12}).
\label{eq:delta_t_cancellation}
\end{equation}

\begin{theorem}[Cartan Decoupling and Exact Isometry]
\label{thm:cartan}
The generators strictly commute: $[\mathbf{K}_3, \mathbf{J}_{12}] = 0$. For any boost rapidity $\varphi \in \mathbb{R}$ and rotation angle $\vartheta \in \mathbb{R}$, the joint state transformation $X' \coloneqq \exp(\varphi \mathbf{K}_3 + \vartheta \mathbf{J}_{12}) X$ evaluates explicitly as the uncoupled $4 \times 4$ Lorentz matrix:
\begin{equation}
\begin{pmatrix} X_0' \\ X_1' \\ X_2' \\ X_3' \end{pmatrix}
=
\begin{pmatrix}
\cosh\varphi & 0 & 0 & \sinh\varphi \\
0 & \cos\vartheta & -\sin\vartheta & 0 \\
0 & \sin\vartheta & \cos\vartheta & 0 \\
\sinh\varphi & 0 & 0 & \cosh\varphi
\end{pmatrix}
\begin{pmatrix} X_0 \\ X_1 \\ X_2 \\ X_3 \end{pmatrix},
\label{eq:cartan_4x4}
\end{equation}
which factorizes into two mutually independent 2D planar maps:
\begin{align}
\begin{pmatrix} X_0' \\ X_3' \end{pmatrix} &= \begin{pmatrix} \cosh\varphi & \sinh\varphi \\ \sinh\varphi & \cosh\varphi \end{pmatrix} \begin{pmatrix} X_0 \\ X_3 \end{pmatrix}, \label{eq:boost} \\
\begin{pmatrix} X_1' \\ X_2' \end{pmatrix} &= \begin{pmatrix} \cos\vartheta & -\sin\vartheta \\ \sin\vartheta & \cos\vartheta \end{pmatrix} \begin{pmatrix} X_1 \\ X_2 \end{pmatrix}. \label{eq:rot}
\end{align}
Furthermore, this mapping preserves the Lorentz quadratic form identically: $\|X'\|_\eta^2 \equiv \|X\|_\eta^2$.
\end{theorem}

\begin{proof}
Direct matrix multiplication confirms that $\mathbf{K}_3 \mathbf{J}_{12} = \mathbf{J}_{12} \mathbf{K}_3 = \mathbf{0}$, hence $[\mathbf{K}_3, \mathbf{J}_{12}] = 0$. By the Baker--Campbell--Hausdorff formula, the matrix exponential factors without series commutator terms:
\begin{equation*}
\exp(\varphi \mathbf{K}_3 + \vartheta \mathbf{J}_{12}) = \exp(\varphi \mathbf{K}_3) \exp(\vartheta \mathbf{J}_{12}).
\end{equation*}
Expanding $\exp(\varphi \mathbf{K}_3) = \mathbf{I} + \sinh\varphi \mathbf{K}_3 + (\cosh\varphi - 1) \mathbf{K}_3^2$ and $\exp(\vartheta \mathbf{J}_{12}) = \mathbf{I} + \sin\vartheta \mathbf{J}_{12} + (1 - \cos\vartheta) \mathbf{J}_{12}^2$ yields the explicit $4 \times 4$ block-diagonal matrix in \eqref{eq:cartan_4x4}, which separates into \eqref{eq:boost} and \eqref{eq:rot}. 

Finally, calculating the Minkowski norm of $X'$:
\begin{align*}
\|X'\|_\eta^2 &= (X_0')^2 - (X_1')^2 - (X_2')^2 - (X_3')^2 \\
&= \left[ (X_0')^2 - (X_3')^2 \right] - \left[ (X_1')^2 + (X_2')^2 \right] \\
&= (X_0^2 - X_3^2)(\cosh^2\varphi - \sinh^2\varphi) - (X_1^2 + X_2^2)(\cos^2\vartheta + \sin^2\vartheta) \\
&= (X_0^2 - X_3^2) - (X_1^2 + X_2^2) = \|X\|_\eta^2,
\end{align*}
establishing exact isometry.
\end{proof}

The full MAN-4D operator evaluates in a single pass:
\begin{equation}
\Psi_{\Delta t}^{\text{4D}}(X) = c + \rho(\Delta t) \cdot \begin{pmatrix}
u_0 \cosh\varphi + u_3 \sinh\varphi \\
u_1 \cos\vartheta - u_2 \sin\vartheta \\
u_1 \sin\vartheta + u_2 \cos\vartheta \\
u_0 \sinh\varphi + u_3 \cosh\varphi
\end{pmatrix}.
\end{equation}

\subsection{Intermediate Formulation: MAN-3D ($\mathbb{H}^2$)}
\label{subsec:man_3d}
In MAN-3D, state triplets evolve in $\mathbb{R}^{1,2} \to \mathbb{H}^2$. Because rotations and boosts do not commute in $\so(1,2)$ ($[J_{12}, K_1] = K_2 \ne 0$), evaluating generic $\Omega = b_1 K_1 + b_2 K_2 + \omega J_{12}$ can be performed via Cayley--Hamilton expansion (Appendix \ref{subsec:so12_expansion}). 
While MAN-3D provides true negative curvature ($K < 0$), its channel divisor ($3$) breaks power-of-two alignment on standard Tensor Core architectures, causing memory slicing overhead in un-fused implementations.

\section{Theoretical Analysis and Systemic Trade-Offs}

\subsection{Elliptic and Hyperbolic Transport: A Symplectic Perspective}
\label{subsec:symplectic_perspective}

MAN-2D parameterizes exact Lorentz boosts by rapidity $\varphi$.
The relationship between the compact circular transport in CTAN
\cite{ananke2026} and the noncompact hyperbolic transport in MAN
can be understood through planar symplectic geometry.
This connection concerns their frozen-parameter conservative
components, rather than their complete dissipative network updates.

\paragraph{Classification within
$\mathrm{Sp}(2,\mathbb{R})=\mathrm{SL}(2,\mathbb{R})$.}
In two dimensions, the real symplectic group coincides with the
special linear group, comprising all linear transformations that
preserve the phase-space area form $dq\wedge dp$.
For $M\in\mathrm{SL}(2,\mathbb{R})$, the characteristic equation is
\begin{equation}
\lambda^2-\operatorname{tr}(M)\lambda+1=0.
\end{equation}
Excluding the central elements $\pm\mathbf{I}$, group elements
are classified into three types:
\begin{itemize}
    \item \textbf{Elliptic elements
    ($|\operatorname{tr}(M)|<2$):}
    Possess complex conjugate eigenvalues on the unit circle
    and are conjugate to Euclidean rotations in $\mathrm{SO}(2)$.

    \item \textbf{Hyperbolic elements
    ($|\operatorname{tr}(M)|>2$):}
    Possess distinct real reciprocal eigenvalues.
    When $\operatorname{tr}(M)>2$, the eigenvalues are
    $e^{\pm\chi}$ with $\chi>0$, and the matrix is conjugate
    to a Lorentz boost in $\mathrm{SO}^+(1,1)$.
    When $\operatorname{tr}(M)<-2$, the eigenvalues are
    $-e^{\pm\chi}$, and the matrix is conjugate to the
    negative of such a boost.
    Nontrivial MAN-2D boosts belong to the positive-trace case.

    \item \textbf{Parabolic elements
    ($|\operatorname{tr}(M)|=2$, $M\ne\pm\mathbf{I}$):}
    Are conjugate to a nontrivial shear or its negative,
    corresponding to trace $2$ or $-2$, respectively.
\end{itemize}

To make the Hamiltonian connection explicit, adopt the convention
\begin{equation}
z=(q,p)^\top,\qquad
\dot z=\mathbf{J}\nabla H(z),\qquad
\mathbf{J}=
\begin{pmatrix}
0&-1\\
1&0
\end{pmatrix}.
\end{equation}
The harmonic Hamiltonian
$H_{\mathrm{ell}}=\frac{1}{2}(q^2+p^2)$ generates rotations,
whereas the saddle Hamiltonian
$H_{\mathrm{hyp}}=\frac{1}{2}(q^2-p^2)$ generates Lorentz boosts:
\begin{equation}
\dot z=\mathbf{J}z
\quad\text{and}\quad
\dot z=\mathbf{K}z,\qquad
\mathbf{K}=
\begin{pmatrix}
0&1\\
1&0
\end{pmatrix},
\end{equation}
respectively.
Their nonzero energy level sets are circles and hyperbolas.
The corresponding exact flows provide the conservative
transport components underlying the circular CTAN construction
and MAN-2D.

\paragraph{Dual-Shear Variants and Their Invariants.}
CTAN \cite{ananke2026} includes a dual-shear variant as an
alternative to exact planar rotation:
\begin{equation}
M_E(s)=
\begin{pmatrix}
1-s^2&-s\\
s&1
\end{pmatrix},
\qquad
\det M_E(s)=1,
\qquad
\operatorname{tr}(M_E(s))=2-s^2.
\end{equation}
For a fixed parameter satisfying $0<|s|<2$, this matrix is
elliptic and preserves the positive-definite quadratic form
\begin{equation}
E_s(u)=u_0^2+s u_0u_1+u_1^2.
\end{equation}
Its invariant curves are therefore generally ellipses rather
than circles. In particular, the dual-shear map is not an exact
Euclidean rotation.

Changing the sign of the second shear yields the hyperbolic
dual-shear update
\begin{align}
u_{1,\mathrm{mid}} &= u_1+s u_0,\\
u_0' &= u_0+s u_{1,\mathrm{mid}},\\
u_1' &= u_{1,\mathrm{mid}},
\end{align}
with transfer matrix
\begin{equation}
M_H(s)=
\begin{pmatrix}
1+s^2&s\\
s&1
\end{pmatrix},
\qquad
\det M_H(s)=1,
\qquad
\operatorname{tr}(M_H(s))=2+s^2.
\end{equation}
For $s\ne0$, this matrix is hyperbolic, with eigenvalues
\begin{equation}
\lambda_\pm=e^{\pm\chi(s)},
\qquad
\chi(s)=2\operatorname{arsinh}\left(\frac{|s|}{2}\right).
\end{equation}
Here $\chi(s)$ describes spectral expansion and contraction;
it is not a boost rapidity with respect to the original
Minkowski metric.

\begin{remark}[Metric Invariance versus Area Preservation]
For a fixed $s$, the dual-shear matrix $M_H(s)$ is symplectic,
but it is not an equivalent implementation of an exact Lorentz
boost. Define
\begin{equation}
q_L(u)=u_0^2-u_1^2,
\qquad
\eta=\operatorname{diag}(1,-1).
\end{equation}
Direct calculation gives
\begin{equation}
q_L(M_H(s)u)-q_L(u)
=
s^2\left[u_0^2+(u_1+s u_0)^2\right]>0
\end{equation}
whenever $s\ne0$ and $u\ne0$.
Thus, $M_H(s)$ does not preserve $\eta$.
Instead, it preserves the modified quadratic form
\begin{equation}
I_s(u)=u_0^2-s u_0u_1-u_1^2,
\end{equation}
since
\begin{equation}
M_H(s)^\top Q_s M_H(s)=Q_s,
\qquad
Q_s=
\begin{pmatrix}
1&-s/2\\
-s/2&-1
\end{pmatrix}.
\end{equation}
Because $\det Q_s=-1-s^2/4<0$, the nonzero level sets of
$I_s$ are hyperbolas.
This invariant depends on $s$ and differs from the Minkowski
quadratic form used in MAN.
\end{remark}

\paragraph{An Algebraic Parameterization of Exact Lorentz Boosts.}
To avoid explicit evaluation of $\cosh\varphi$ and $\sinh\varphi$
while retaining the Lorentz constraint, one may instead
parameterize the boost as
\begin{equation}
B(a)=
\begin{pmatrix}
\sqrt{1+a^2}&a\\
a&\sqrt{1+a^2}
\end{pmatrix},
\qquad a\in\mathbb{R}.
\end{equation}
In exact arithmetic,
\begin{equation}
B(a)^\top\eta B(a)=\eta,\qquad
\det B(a)=1,\qquad
B_{00}(a)>0,
\end{equation}
so $B(a)\in\mathrm{SO}^+(1,1)$ for every real $a$.
It is exactly the Lorentz boost with rapidity
$\varphi=\operatorname{arsinh}(a)$, although evaluating $B(a)$
does not require computing this inverse function.
When $a$ is predicted directly, the matrix requires only
elementary arithmetic and one square root.

Combined with the original radial relaxation, this
parameterization retains the exact log-radius contraction
and invariant-hyperboloid properties of the frozen-parameter
conditional flow.
It changes the boost parameterization, not its geometry.

\begin{figure}[ht]
    \centering
    \includegraphics[width=0.92\linewidth]
    {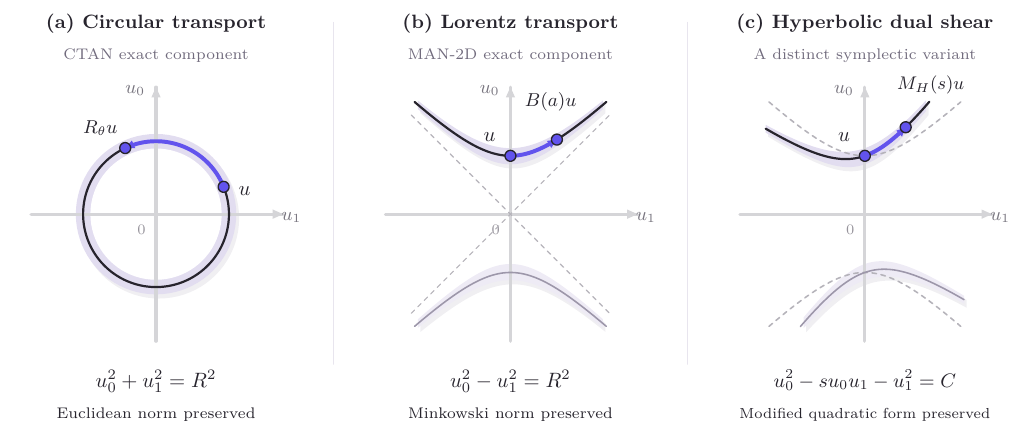}
    \caption{
    Elliptic and hyperbolic transport with frozen parameters.
    (a) Exact circular transport preserves the Euclidean norm.
    (b) Exact Lorentz transport, including the algebraic
    parameterization $B(a)$, preserves the Minkowski norm.
    (c) The hyperbolic dual-shear map preserves the modified
    quadratic form $I_s(u)=u_0^2-su_0u_1-u_1^2$, rather than
    the original Minkowski norm; gray dashed curves show
    the original Minkowski level set through the initial point.
    Panel (c) uses $s=0.7$, and the highlighted endpoints satisfy
    $u'=M_H(s)u$. The connecting arc indicates the preserved
    quadratic level set, not the intermediate path of the
    two shear substeps.
    Shaded bands indicate neighboring levels.
    Only the conservative components are illustrated;
    radial dissipation and input-dependent conditioning are excluded.
    }
    \label{fig:symplectic_perspective}
\end{figure}

Figure~\ref{fig:symplectic_perspective} illustrates the distinction between symplecticity and metric preservation.
For fixed parameters, exact circular transport, exact Lorentz transport, and the hyperbolic dual-shear map all preserve
phase-space area, but they preserve different quadratic forms. Circular transport follows Euclidean norm contours, whereas
Lorentz transport follows Minkowski norm contours. The algebraic boost $B(a)$ retains the latter geometry and
therefore remains compatible with the exact radial relaxation derived above.
By contrast, the dual-shear map follows the level sets of a parameter-dependent quadratic form $I_s$ and generally changes
the Minkowski radius. Thus, the common symplectic structure explains the algebraic
relationship between these conservative maps.

\subsection{Comparison of Geometric Scaffolds}
Table \ref{tab:comparison_theory} compares the structural and computational properties of the geometric scaffolds across channel capacity $D$.

\begin{table}[h]
\centering
\small
\caption{Systemic comparison of geometric scaffolds across channel capacity $D$ (per-factor geometry).}
\label{tab:comparison_theory}
\begin{tabular}{lcccc}
\toprule
\textbf{Metric / Property} & \textbf{2D Torus} & \textbf{MAN-2D} & \textbf{MAN-3D} & \textbf{MAN-4D} \\
\midrule
Ambient Space & $\mathbb{R}^2$ & $\mathbb{R}^{1,1}$ & $\mathbb{R}^{1,2}$ & $\mathbb{R}^{1,3}$ \\
Attractor & $\mathbb{S}^1$ & $\mathbb{H}^1$ & $\mathbb{H}^2$ & $\mathbb{H}^3$ \\
Factor Curvature $K$ & Flat (1D) & Flat (1D) & $-1/R^2$ & $\mathbf{-1/R^2}$ \\
Factor Ball Volume & $\min(2r, 2\pi R)$ & $2r$ & $\sim e^{r/R}$ & $\mathbf{\sim e^{2r/R}}$ \\
Spatial Rotation & $\SO(2)$ & None & $\SO(2)$ & $\mathbf{\SO(2)}$ \\
Channel Divisor & $2$ & $\mathbf{2}$ & $3$ & $\mathbf{4}$ \\
Standard Channel Fit & Yes & \textbf{Yes} & No & \textbf{Yes} \\
Blocks Count $M$ & $D/2$ & $\mathbf{D/2}$ & $D/3$ & $\mathbf{D/4}$ \\
Control Params / Block & $4$ & $4$ & $7$ & $7$ \\
Context Output Width & $2.00 D$ & $2.00 D$ & $2.33 D$ & $\mathbf{1.75 D}$ \\
Throughput / Speed & Fastest & Fast & Slowest & Fast \\
\bottomrule
\end{tabular}
\end{table}

\paragraph{Filter Granularity vs. Manifold Curvature.}
Table \ref{tab:comparison_theory} highlights the structural trade-off governing MAN architectures:
\begin{itemize}
    \item \textbf{MAN-2D (Granularity Optimum):} Extrinsically, $\SO^+(1,1)$ provides non-compact coordinate expansion, avoiding periodic phase-locking. Partitioning channels into $D/2$ blocks yields maximum independent filtering granularity, which proves optimal for fine-grained feature extraction on low-resolution image benchmarks.
    \item \textbf{MAN-4D (Symmetry and Compactness Optimum):} MAN-4D unlocks true per-factor negative curvature ($K < 0$, volume scaling $e^{2r/R}$) while saving parameter overhead ($1.75D$ vs. $2.0D$). It simultaneously models hierarchical depth progression and continuous transverse category rotation ($\mathbb{S}^2$ horizon), maintaining power-of-two Tensor Core alignment.
\end{itemize}

\paragraph{Parameter Accounting.} In all variants, $\beta$ is a learned per-block parameter. The context generator dynamically outputs 4, 7, and 7 scalars per block for 2D, 3D, and 4D models respectively (predicting centers $c$, radius offsets $\Delta \ln R$, and tangential velocities $\varphi, \vartheta$). Because 4D groups channels into $D/4$ blocks, its context projection layer is narrower ($1.75D$) than 2D ($2.0D$).


\section{Experiments}

\subsection{Experimental Setup}
We evaluate MAN on CIFAR-100 \cite{krizhevsky2009learning} (50,000 training, 10,000 testing images across 100 fine-grained categories at $32 \times 32$ resolution). All models are trained completely from scratch without external pretraining under standardized settings: AdamW optimizer, cosine annealing schedule, and 200 epochs.

\paragraph{Model Tier Configurations.}
We instantiate our architectures across two scalable pyramidal depth tiers. 
The compact tier (\textbf{Tier-1}, $\sim 1\text{M}$) adopts stage depths $(3, 4, 5)$ and channel dimensions $(96, 128, 160)$; under this budget, MAN-2D-1 comprises $\sim 1.02\text{M}$ parameters, whereas MAN-4D-1 comprises $\sim 0.97\text{M}$ parameters, saving $\sim 56\text{K}$ parameters due to its narrower context generator output width ($1.75D$ vs. $2.0D$). To evaluate architectural scaling under deeper cascades, the scaled tier (\textbf{Tier-2}, $\sim 2.1\text{M}$) expands stage depths to $(3, 6, 9)$ and channel dimensions to $(96, 144, 192)$, yielding MAN-2D-2 with $\sim 2.13\text{M}$ parameters and MAN-4D-2 with $\sim 2.01\text{M}$ parameters. 
Additionally, to isolate the impact of odd-dimensional manifolds, we implement the intermediate \textbf{MAN-3D} baseline under depths $(3, 4, 5)$ and channel dimensions $(96, 144, 192)$, totaling $\sim 1.42\text{M}$ parameters.

\subsection{Main Benchmark Results}
Table \ref{tab:cifar100_main} summarizes top-1 accuracy and parameter efficiency on CIFAR-100.

\begin{table}[h]
\centering
\small
\caption{Top-1 classification accuracy on CIFAR-100 without pretraining. Best overall results in bold; best sub-1M results underlined.}
\label{tab:cifar100_main}
\begin{tabular}{lcc}
\toprule
\textbf{Architecture} & \textbf{Parameters} & \textbf{Top-1 Accuracy} \\
\midrule
MobileNetV2 \cite{sandler2018mobilenetv2} & 2.30M & 70.90\% \\
ShuffleNetV2 1.5$\times$ \cite{ma2018shufflenet} & 2.60M & 75.95\% \\
ResNet-18 \cite{he2016deep} & 11.20M & 76.75\% \\
ResNet-50 \cite{he2016deep} & 23.70M & 79.14\% \\
DenseNet-121 \cite{huang2017densely} & 7.00M & 78.50\% \\
CliffordNet-1 \cite{ji2026cliffordnet} & 1.40M & 77.82\% \\
CliffordNet-2 \cite{ji2026cliffordnet} & 2.60M & 79.05\% \\
\midrule
\multicolumn{3}{l}{\textit{Flat Product-Torus Baselines (CTAN)}} \\
CTAN-Hier-1 \cite{ananke2026} & 0.91M & 78.41\% \\
CTAN-Hier-2 \cite{ananke2026} & 1.89M & 79.88\% \\
CTAN-Hier-3 \cite{ananke2026} & 2.13M & 80.32\% \\
\midrule
\multicolumn{3}{l}{\textit{Minkowski Attractor Networks (Ours)}} \\
\textbf{MAN-4D-1} (Spacetime Cartan) & \textbf{0.97M} & \underline{80.80\%} \\
\textbf{MAN-2D-1} (Granular Workhorse) & \textbf{1.02M} & \textbf{81.03\%} \\
MAN-3D (Intermediate $\mathbb{H}^2$) & 1.42M & 81.01\% \\
\textbf{MAN-4D-2} (Scaled Cartan) & \textbf{2.01M} & 81.75\% \\
\textbf{MAN-2D-2} (Scaled Granular) & \textbf{2.13M} & \textbf{81.82\%} \\
\bottomrule
\end{tabular}
\end{table}

\begin{enumerate}
    \item \textbf{The Minkowski Leap over Flat Tori:} Across both scale tiers, transitioning from flat tori to Minkowski spaces produces a consistent gain from $\sim 78.41\%$ to $\sim 81.03\%$ in Tier-1 and from $80.32\%$ to $81.82\%$ in Tier-2. This empirically confirms that non-compact Minkowski dissipation provides a superior inductive bias compared to compact tori.
    \item \textbf{Effectiveness of MAN-2D:} MAN-2D-1 attains the highest accuracy in Tier-1 (\textbf{81.03\%}) with the highest training throughput, confirming that maximizing channel filter granularity ($D/2$ blocks) is highly effective for visual feature extraction. Scaling to Tier-2, MAN-2D-2 reaches \textbf{81.82\%}, outperforming 23.7M ResNet-50 by \textbf{+2.68\%} with $11\times$ fewer parameters.
    \item \textbf{Parameter Efficiency of MAN-4D:} MAN-4D-1 attains \textbf{80.80\%} with only 0.97M parameters (saving $\sim 56\text{K}$ parameters over MAN-2D-1 due to $1.75D$ context overhead). In Tier-2, MAN-4D-2 scales to \textbf{81.75\%} at 2.01M parameters, demonstrating the parameter compactness of the Cartan spacetime formulation.
\end{enumerate}

\subsection{Dimensional Scaling}
\label{subsec:ablation_dims}

To isolate the representational and computational impact of the ambient manifold dimensionality, Table~\ref{tab:ablation_dims} systematically benchmarks the four geometric scaffolds under matched Tier-1 network depths $(3, 4, 5)$.

\begin{table}[h]
\centering
\small
\caption{Systematic dimensional ablation on CIFAR-100 across geometric scaffolds under matched Tier-1 depths $(3, 4, 5)$.}
\label{tab:ablation_dims}
\begin{tabular}{lcccc}
\toprule
\textbf{Model} & \textbf{Manifold} & \textbf{Params} & \textbf{Channels} & \textbf{Top-1} \\
\midrule
CTAN-2D & $\mathbb{S}^1$ & 0.91M & $(96, 128, 160)$ & 78.41\% \\
MAN-2D-1 & $\mathbb{H}^1$ & 1.02M & $(96, 128, 160)$ & \textbf{81.03\%} \\
MAN-3D & $\mathbb{H}^2$ & 1.42M & $(96, 144, 192)$ & 81.01\% \\
MAN-4D-1 & $\mathbb{H}^3$ & \textbf{0.97M} & $(96, 128, 160)$ & \underline{80.80\%} \\
\bottomrule
\end{tabular}
\end{table}

Several profound architectural insights emerge from this cross-manifold comparison:

\paragraph{The Categorical Minkowski Leap.} 
Most prominently, transitioning from the flat compact torus ($\mathbb{S}^1$, 78.41\%) to any of the Minkowski hyperbolic formulations ($\mathbb{H}^1, \mathbb{H}^2, \mathbb{H}^3$) yields an immediate, categorical accuracy surge of $+2.39\%$ to $+2.68\%$, elevating performance from the 78\% tier to the 81\% plateau. 
This empirical leap confirms our foundational hypothesis: compact periodic orbits in flat tori suffer from phase-locking and bounded coordinate congestion under dense classification, whereas non-compact Minkowski flows provide unconstrained monotonic coordinate dilation alongside scale-invariant logarithmic dissipation, offering a fundamentally superior inductive prior for visual representation learning.

\paragraph{Granularity vs. Geometric Dimensionality (2D vs. 4D).} 
Comparing MAN-2D-1 and MAN-4D-1 under identical stage channel dimensions $(96, 128, 160)$ highlights the trade-off between coordinate factorization granularity and manifold dimensionality:
\begin{itemize}
    \item \textbf{MAN-2D-1 (81.03\%):} Factors channels into $K = D/2$ independent 2D Minkowski planes (deploying 80 independent centers and radii in Stage 3). This fine-grained channel decoupling acts as a dense bank of independent non-linear bandpass filters. On low-resolution visual tokens ($32 \times 32$), this localized flexibility yields the highest classification accuracy (\textbf{81.03\%}) with the fastest execution throughput, as its $2\times 2$ Lorentz boost requires minimal tensor slicing and zero spatial rotation overhead.
    \item \textbf{MAN-4D-1 (80.80\%):} Groups channels into $M = D/4$ blocks (40 blocks in Stage 3). Crucially, by virtue of the Cartan maximal decomposition, MAN-4D requires only $1.75D$ context generator output channels (versus $2.0D$ in 2D), shrinking total network capacity to merely \textbf{0.97M parameters} (saving $\sim 56\text{K}$ parameters). Despite operating with fewer independent blocks and lower parameter capacity, MAN-4D retains a competitive 80.80\%, demonstrating the representational strength of coupling longitudinal depth boosts with continuous transverse $\mathrm{SO}(2)$ category rotations.
\end{itemize}

\paragraph{The Cost of Odd-Dimensional Embeddings (MAN-3D).} 
MAN-3D ($\mathbb{H}^2$) validates the empirical power of intrinsic negative curvature ($K < 0$), achieving 81.01\%. However, because its channel divisor is 3, preserving representational balance necessitated expanding channels to non-standard widths $(96, 144, 192)$, which inflated parameters to 1.42M without surpassing the 1.02M MAN-2D-1 baseline. Furthermore, evaluating non-commuting generators and arbitrary-direction vector projections in 3D incurs intermediate memory traffic, resulting in the slowest training throughput. This confirms that while 3D hyperbolic geometry is theoretically sound, 2D and 4D Minkowski formulations represent the true Pareto-optimal architectures for modern hardware.


\section{Conclusion}
We introduced Minkowski Attractor Networks (MAN), an operator-splitting-inspired continuous-dynamical framework grounded in pseudo-Riemannian Minkowski spacetime. MAN combines cone lifting with conditional Lorentz transport and closed-form radial relaxation, providing a noncompact geometric inductive bias without iterative ODE integration. Our primary visual workhorse, MAN-2D, maximizes channel filter granularity to achieve state-of-the-art parameter efficiency, while our 4D extension, MAN-4D, evaluates exact closed-form spacetime isometries via commuting Cartan subalgebras. Reaching up to \textbf{81.82\%} accuracy on CIFAR-100 with only 2.13M parameters, MAN demonstrates that closed-form Minkowski spacetime dynamics provide a powerful, mathematically rigorous foundation for next-generation visual architectures.


\bibliographystyle{unsrt} 
\bibliography{man}  

\clearpage

\appendix
\section{Appendix}

\subsection{Closed-Form Formulation of Full 6-DOF Lorentz Dynamics}
\label{subsec:full_6d_lorentz}
For completeness, the unconstrained 6-DOF Lorentz transformation \eqref{eq:full_polar_decomp} evaluates analytically without series truncation:
\begin{align}
\mathbf{X}_{\text{rot}} &= \mathbf{X}\cos\vartheta + (\mathbf{k} \times \mathbf{X})\sin\vartheta + \mathbf{k}(\mathbf{k}\cdot\mathbf{X})(1 - \cos\vartheta), \\
X_0' &= X_0 \cosh\varphi + (\mathbf{X}_{\text{rot}} \cdot \mathbf{n})\sinh\varphi, \\
\mathbf{X}' &= \mathbf{X}_{\text{rot}} + \left[ X_0 \sinh\varphi + (\mathbf{X}_{\text{rot}} \cdot \mathbf{n})(\cosh\varphi - 1) \right] \mathbf{n},
\end{align}
where $\mathbf{k} = \boldsymbol{\vartheta}/\|\boldsymbol{\vartheta}\|$, $\mathbf{n} = \boldsymbol{\varphi}/\|\boldsymbol{\varphi}\|$, and $\vartheta = \|\boldsymbol{\vartheta}\|, \varphi = \|\boldsymbol{\varphi}\|$. While mathematically exact and metric-preserving, evaluating this operator across deep cascades introduces higher latency than the factorized Cartan flow.

\subsection{Cayley--Hamilton Expansion for $\so(1,2)$}
\label{subsec:so12_expansion}
For any $\Omega \in \so(1,2)$ with spatial rotation $\omega$ and boosts $(b_1, b_2)$, let $\kappa = b_1^2 + b_2^2 - \omega^2$. By the Cayley--Hamilton theorem, $\Omega^3 = \kappa \Omega$. The exact exponential can be evaluated without infinite series truncation as:
\begin{equation}
e^{t\Omega} = \mathbf{I}_3 + A_\kappa(t) \Omega + B_\kappa(t) \Omega^2,
\end{equation}
where $A_\kappa(t) = \frac{\sinh(t\sqrt{\kappa})}{\sqrt{\kappa}}$ for $\kappa > 0$, $A_0(t) = t$ for $\kappa = 0$, and $A_\kappa(t) = \frac{\sin(t\sqrt{-\kappa})}{\sqrt{-\kappa}}$ for $\kappa < 0$. To prevent numerical cancellation as $\kappa \to 0$, $B_\kappa(t)$ is stably evaluated via:
\begin{equation}
B_\kappa(t) = \begin{cases}
\frac{2\sinh^2(t\sqrt{\kappa}/2)}{\kappa}, & \kappa > 0, \\
\frac{t^2}{2}, & \kappa = 0, \\
\frac{2\sin^2(t\sqrt{-\kappa}/2)}{-\kappa}, & \kappa < 0.
\end{cases}
\end{equation}

\end{document}